%% file: whylowfreq.tex
\documentclass[11pt,a4paper]{article}
\usepackage[colorlinks=true,linkcolor=blue]{hyperref}
\input{header.en.tex}

\newcommand{\Bol}{\mathrm{Bolt}_{\pi_0}}
\newcommand{\piref}{{\pi_0}}
\newcommand{\AK}{\mathcal{A}}
\newcommand{\drho}{\grave \rho}

\title{Which Features are Best for Successor Features?}
\date{}
\author{Yann Ollivier}

\newcommand{\TODO}[1]{{\color{red} TODO: {#1}}}
\newcommand{\todo}[1]{\TODO{#1}}
\newcommand{\option}[1]{{\color[rgb]{.4,0,.8}[Optional:#1]}} %blue for working version: optional stuff
\renewcommand{\TODO}[1]{}
\renewcommand{\todo}[1]{}
\renewcommand{\option}[1]{}  %remove optional text

\begin{document}

\maketitle

\begin{abstract}
In reinforcement learning, universal successor features (SFs) are a way to provide zero-shot adaptation to new tasks at test time: they provide optimal policies for all downstream reward functions lying in the linear span of a set of base features. But it is unclear what constitutes a good set of base features, that could be useful for a wide set of downstream tasks beyond their linear span. Laplacian eigenfunctions (the eigenfunctions of $\Delta+\Delta^\ast$ with $\Delta$ the Laplacian operator of some reference policy and $\Delta^\ast$ that of the time-reversed dynamics) have been argued to play a role, and offer good empirical performance.

Here, for the first time, we identify the optimal base features based on an objective criterion of downstream performance, in a non-tautological way without assuming the downstream tasks are linear in the features. We do this for three generic classes of downstream tasks: reaching a random goal state, dense random Gaussian rewards, and random ``scattered'' sparse rewards.  The features yielding optimal expected downstream performance turn out to be the \emph{same} for these three task families.  They do not coincide with Laplacian eigenfunctions in general, though they can be expressed from $\Delta$: in the simplest case (deterministic environment and decay factor $\gamma$ close to $1$), they are the eigenfunctions of $\Delta^{-1}+(\Delta^{-1})^\ast$.
% : namely, the features whose
% \emph{advantage functions} have the largest norm in the given environment.
% 
% Furthermore, we get a fully explicit characterization of these functions for deterministic
% environments, depending on the decay factor $\gamma$. For $\gamma\to 1$,
% the optimal features are the largest eigenfunctions of
% $\Delta^{-1}+(\Delta^{-1})^\ast$: like Laplacian eigenfunctions, these encode the idea of
% capturing long-range, low-frequency behavior, but, in general, differ
% from Laplacian eigenfunctions unless $\Delta$ is self-adjoint (the reference
% policy is time-reversible). In contrast, for small $\gamma$, the optimal
% features are the \emph{smallest singular functions} of the transition
% matrix, corresponding to short-range, high-frequency behavior.

We obtain these results under an assumption of large behavior cloning regularization with respect to a reference policy, a setting often used for offline RL. Along the way, we get new insights into KL-regularized\option{natural} policy gradient, and into the lack of SF information in the norm of Bellman gaps.
\end{abstract}

\section{Introduction and Related Work}

The successor features (SFs) framework holds the promise of solving any
new reinforcement learning task in a fixed environment in an almost
zero-shot manner, without extensive learning or planning for each new
task (see e.g., \cite{borsa2018universal,zeroshot,allpolicies} and the
references therein). At train time, the agent observes reward-free
transitions in an environment, and learns some features and a parametric
family of policies. At test time, the agent is faced with a new task
specified via a reward function $r$. The function $r$ is linearly
projected onto the set of features, and then a suitable pre-trained
policy is applied. At test time, the linear projection requires no
learning or planning, and only uses a relatively small number of reward
values (or knowledge of the reward function itself, e.g., for
goal-reaching).

A good choice of features is crucial for successor features, and several
approaches to select relevant features for SFs have been proposed.  Given
a finite number of features $\phi$, SFs can produce policies within a
family of tasks directly related to $\phi$
\cite{barreto2017successor,borsa2018universal,zhang2017deep,grimm2019disentangled},
but this often uses hand-crafted $\phi$ or features $\phi$ that linearize
some known training rewards.  VISR and APS
\cite{hansen2019fast,liu2021aps} build $\phi$ automatically via diversity
criteria \cite{eysenbach2018diversity,gregor2016variational}.
Forward-backward representations \cite{allpolicies,zeroshot} use
successor measures \cite{successorstates} to learn features $\phi$ that
model the main variations of the distribution of visited states depending
on the starting point and policy.  In related contexts, further spectral
variants have been proposed to provide a basis of features to express
reward functions and $Q$-functions: the singular value decomposition of
the inverse Laplacian, the singular value decomposition of the transition
matrix $P_{\pi_0}$ induced by an exploration policy $\pi_0$, the
eigenfunctions of $P_{\pi_0}+ P^\ast_{\pi_0}$, and more (see
\cite{ghosh2020representations, ren2022spectral} and the references
therein).

\cite{zeroshot} compared SFs with features built from a number of
approaches: Laplacian
eigenfunctions \cite{wu2018laplacian,mahadevan2007proto},
forward-backward representations, APS,
auto-encoders, inverse
dynamics models, spectral decompositions of the transition matrix, and
more.
Forward-backward representations and Laplacian eigenfunctions were found
to perform best on average.

\bigskip

Yet all these choices of features are based on somewhat heuristic
arguments and discussions. Here, for the first time, we fully
characterize the best features for SFs mathematically, by directly
optimizing expected downstream performance.

We define three agnostic models of downstream tasks: random Gaussian reward
functions; reaching a random goal state; and ``scattered'' random rewards
(a random number of sparse rewards placed at randomly located states, with random
signs and magnitudes). We then ask which features provide the best
expected performance on these tasks.

Surprisingly, the conclusions are identical for all three classes: \emph{even
though these three models cover very different tasks} (dense rewards,
single-state rewards, multiple sparse rewards), \emph{the best features are the
same}. Thus, the conclusions are robust to the precise choice of a
downstream task model. The existence of clear optimal features for these
tasks is itself a nontrivial result, given the prior-free nature of these
reward models.

\paragraph{Overview of results.} All along, following previous work for
offline reinforcement learning (see e.g.\ the survey
\cite{levine2020offline}), we work with entropy-regularized policies that
stay relatively close to a reference policy $\pi_0$.  We estimate the
\emph{regularized return} $G^{\pi}_r$ of a policy $\pi$ for a reward $r$,
which includes a Kullback--Leibler penalty for deviating from a $\pi_0$
(Definition~\ref{def:regul}).  We assume that the regularization constant
(temperature)
$T$ is relatively large, and derive optimal features up to an error
$O(1/T^2)$.

Namely, we look for the features $\phi$ such that, when using the policy $\hat \pi$ estimated by successor features (Definition~\ref{def:rsf}), the expected regularized return
\begin{equation}
\E_r [G^{\hat \pi}_r]
\end{equation} of $\hat \pi$ is maximal. The expectation is over rewards $r$ in three simple models of
downstream tasks (Section~\ref{sec:rmodels}). Our main findings are as follows. All results are up to
an error $O(1/T^2)$ on optimality.

\begin{itemize}
\item In deterministic environments, the optimal features for regularized
successor features are the largest eigenfunctions of
\begin{equation}
\Delta^{-1}+(\Delta^{-1})^{\ast}-(1-\gamma^2)
(\Delta^{-1})^\ast \Delta^{-1}
\end{equation}
(Theorem~\ref{thm:gen}),
where $\Delta\deq \Id-\gamma P_{\pi_0}$ is the Laplacian operator of the
reference policy $\pi_0$,
and $(\Delta^{-1})^\ast$
is the adjoint of $\Delta^{-1}$ acting on
$L^2(\rho)$. \footnote{This adjoint is
$\diag(\rho)^{-1}\transp{(\Delta^{-1})}\diag(\rho)$, not
$\transp{(\Delta^{-1})}$.} Here $\rho$ is the stationary distribution of
state-actions under
the reference policy.

\item
For $\gamma\to 1$ this reduces to the largest eigenfunctions of
$\Delta^{-1}+(\Delta^{-1})^{\ast}$ (Theorem~\ref{thm:lap}), corresponding
to long-range (low frequency) information on the reward function.

\item For $\gamma\to 0$ this reduces to the smallest eigenfunctions of
$P_{\pi_0}^\ast P_{\pi_0}$ (Proposition~\ref{prop:gamma0}), corresponding
to short-range (high-frequency) information on the reward function.

\item For general, non-deterministic environments, 
we show that the optimal features are the reward
functions $r$ whose \emph{advantage functions} have maximal norm for a
given norm of $r$ (Theorem~\ref{thm:featurereturn} and
Corollary~\ref{cor:optfeatures}). 
\option{More precisely, the optimal features are the largest
eigenfunctions of the self-adjoint operator
$\diag(\rho)^{-1}\AK_{\pi_0}$, where $\AK_{\pi_0}$ is the unique
symmetric positive semidefinite matrix such that for any reward $r$, the
norm of its advantage function $A^{\pi_0}_r$ is
$\norm{A^{\pi_0}_r}^2_{L^2(\rho)}=\transp{r} \AK_{\pi_0} r$
(Definition~\ref{def:AK}).}

Explicitly, the optimal features are the
largest eigenfunctions of the matrix
\begin{equation}
\diag(\rho)^{-1} \transp{(\Delta^{-1})}\left(
\diag(\rho)-\transp{\pi_0}\diag(\rho_S)\pi_0\right)\Delta^{-1}.
\end{equation}
where again, $\rho$ is the stationary distribution of state-actions under
$\pi_0$, and $\rho_S$ is its marginal on states.
In deterministic environments, this expression simplifies to 
the results above.

\item We show that, in successor features, the average norm of Bellman
gaps for downstream tasks is uninformative as to
which features perform best: it only depends on the number of linearly
independent features
(Proposition~\ref{prop:gap}).

\item As a key intermediate result, we show that for KL-penalized policy
improvement, the optimality gap due to imperfect $Q$-function estimation
is equal to the norm of the error on the advantage function
(Theorem~\ref{thm:return}).

\option{
\item Finally, the advantage kernel gives the performance improvement along natural policy gradient
trajectories (Theorem \TODO{}): if $(\pi_t)_{t\geq 0}$ is the family of
policies learned by natural policy gradient on a reward $r$, then
\begin{equation}
\frac{\d}{\d t} \E_{\rho_t} [r]=\transp{r}\AK_{\pi_t} r
\end{equation}
where $\rho_t$ is the invariant distribution of $\pi_t$.}

\end{itemize}

\paragraph{Comparison with Laplacian eigenfunctions and forward-backward
representations.}
These results complement the empirical results in
\cite{zeroshot}, where forward-backward representations and Laplacian
eigenfunctions performed best among the SF methods tested.

The Laplacian operator of a policy $\pi_0$
is $\Delta=\Id -\gamma P_{\pi_0}$, and Laplacian eigenfunctions are
the smallest eigenfunctions of
$\Delta+\Delta^\ast$. 
Forward-backward representations learn a
finite-rank approximation of $\Delta^{-1}$ as
$\transp{F}B$ and then use $B$ as features. Here we find the optimal
features to be the eigenfunctions of $\Delta^{-1}+(\Delta^{-1})^{\ast}$
(in the simplest case of a deterministic environment and $\gamma\to 1$).
In general $\Delta^{-1}+(\Delta^{-1})^{\ast}\neq (\Delta+\Delta^\ast)^{-1}$ except
in very specific environments (see discussion after
Theorem~\ref{thm:lap}). So in general, the optimal features derived here
differ from Laplacian eigenfunctions. They also differ from $B$ in
the forward-backward representation: they would be closer to extracting the
symmetric part of the finite-rank approximation $\transp{F}B$.

Still, our results may explain the good empirical performance of
forward-backward representations and Laplacian eigenfunctions for SFs:
they are the methods that come closest to the theoretically optimal
features among the methods tested in \cite{zeroshot}.

\paragraph{Limitations and perspectives.} A first limitation of these
results comes from the entropy regularization: all the statements about
optimality hold up to an error $O(1/T^2)$. It is not clear how far this
approximation extends in practice. Still, a regularized setup is natural
for zero-shot reinforcement learning based on a fixed
foundation model such as successor feature: the policies deployed at test
time only depend on information from the trainset used to build the
model, so it makes sense not to deviate too much from the states and
behaviors explored in the trainset \cite{levine2020offline}.

Second, in this work, we provide no algorithms to actually learn the
optimal features (either the eigenfunctions of
$\Delta^{-1}+(\Delta^{-1})^\ast$, or the reward functions $r$ that
maximize the norm of the advantage function). This is left to future
work.

Finally, more fundamental limitations come from the setup of successor
features itself. SFs rely on the linear projection of the reward onto a
subspace of features. This method can only be exact in a linear subspace
of reward functions, and the task encoding is linear. Recent works such
as \TODO{} attempt to remove this limitation by defining
\emph{auto-regressive features} that let finer task encoding features
depend on previously computed, coarse task encoding features, resulting
in a fully nonlinear task encoding. SFs are also limited to tackling new
tasks within a given environment and dynamics, not new environment or
different dynamics in the same environment; some workarounds have been
proposed to handle similar enough environments
\cite{zhang2017deep,abdolshah2021new}.

\TODO{is there work on kernelized SFs?}

\paragraph{Acknowledgments.} The author would like to thank Andrea
Tirinzoni and Ahmed Touati for comments on the text, suggestions for
references, and
helpful discussions around this topic.

\section{Preliminaries: Notation, Universal Successor Features,
Regularized Policies}

\subsection{Setup and Notation}

\paragraph{Markov decision processes, matrix notation.}
Let $\mathcal{M}=(S,A,P,\gamma)$ be
a reward-free Markov decision process (MDP)
with state space $S$, action space $A$, transition probabilities
$P(s'|s,a)$ from state $s$ to $s'$ given action $a$,
and discount factor $0 < \gamma < 1$ \cite{sutton2018reinforcement}. We
assume that 
$S$ and $A$ are finite. \footnote{In principle, most of the ideas in this text
extend to continuous spaces, but this would make the statements
unnecessarily technical.}
A policy $\pi$ is a function $\pi\from S\to \mathrm{Prob}(A)$ mapping a
state $s$ to the probabilities of actions in $A$.
Given
$(s_0,a_0)\in S\times A$ and a policy $\pi$,
we denote $\E[\cdot|s_0,a_0,\pi]$ the
expectations under state-action sequences $(s_t,a_t)_{t
\geq 0}$ starting at $(s_0,a_0)$ and following policy $\pi$ in the
environment, defined by sampling $s_t\sim P( s_t|s_{t-1},a_{t-1})$ and
$a_t\sim \pi( a_t | s_t)$. We define $P_\pi(s', a' | s, a)
\deq
P( s'| s, a) \pi( a' | s')$, 
the state-action transition probabilities 
induced by $\pi$.  Given a reward
function $r\from S\times A \to \R$, the $Q$-function of $\pi$ for $r$ is
$Q_r^\pi(s_0,a_0)\deq \sum_{t\geq 0} \gamma^t \E
[r(s_t,a_t)|s_0,a_0,\pi]$. The \emph{value function} of $\pi$ for $r$ is
$V_r^\pi(s)\deq \E_{a\sim \pi(s)} Q_r^\pi(s,a)$, and the
\emph{advantage function} is
is $A_r^\pi(s,a)\deq Q_r^\pi(s,a)-V_r^\pi(s)$.

For the statements and proofs, it is convenient to view these objects as
vectors and matrices. We treat rewards and $Q$-functions as column
vectors of size $\#S\times \#A$. We view $P$ as
a matrix of size $(\#S\times \#A)\times \#S$ with entries
$P_{(sa)s'}=P(s'|s,a)$. A policy $\pi$ is seen as a matrix of size
$\#S\times (\#S\times \#A)$ with entries
$\pi_{s(s'a)}=\pi(a|s)\1_{s=s'}$. 
We denote $P_{\pi}\deq P\pi$.
With these conventions, the
$Q$-function satisfies the Bellman equation
\begin{equation}
Q^\pi_r=r+\gamma P_\pi Q^{\pi}_r.
\end{equation}

\paragraph{Reference policy $\pi_0$, invariant distribution $\rho$.} Let $\pi_0$ be some reference policy.
We assume
that $\pi_0$ is ergodic. Let $\rho$ be the asymptotic distribution of
state-actions induced by $\pi_0$ (the stationary distribution of
$P_{\pi_0}$). Typically, $\pi_0$ is an exploration policy used to build
some training set for RL algorithms; in this situation, the distribution
of states and actions in the training set is approximately $\rho$.
Let $\rho_S$ be the marginal distribution of states under $\rho$.
We abbreviate
\begin{equation}
\drho \deq \diag(\rho),\qquad \drho_S\deq \diag(\rho_S)
\end{equation}
the diagonal matrices of size $\#S\times \#A$ with diagonal entries equal
to $\drho_{(sa)(sa)}=\rho(s,a)$, and of size $\#S$ with diagonal entries
equal to $\rho_S$, respectively.

\paragraph{$L^2(\rho)$ norm, advantage norm.} Two
for functions $f$ on
$S\times A$ will play an important role: the $L^2(\rho)$ norm, defined as
\begin{equation}
\norm{f}^2_{L^2(\rho)}\deq\E_{(s,a)\sim \rho} [f(s,a)^2]=\transp{f}\drho f
\end{equation}
and the \emph{advantage norm}, defined as
\begin{equation}
\norm{f}^2_A\deq \E_{(s,a)\sim \rho}\left[\left(
f(s,a)-\E_{a'\sim \pi_0(s)} [f(s,a')]
\right)^2\right].
\end{equation}
By construction,
\begin{equation}
\norm{Q^{\pi_0}_r}_A=\norm{A^{\pi_0}_r}_{L^2(\rho)}
\end{equation}
though this does not hold for other policies $\pi\neq \pi_0$, because
the expectation inside $\norm{f}^2_A$ is with respect to $\pi_0$.

We denote $\langle\cdot,\cdot\rangle_{L^2(\rho)}$ and
$\langle\cdot,\cdot\rangle_A$ the inner products associated
with these two norms.

The \emph{adjoint} of a linear operator $M$ on $L^2(\rho)$ is the unique
operator $M^\ast$ such that $\langle x,My\rangle_{L^2(\rho)}=\langle
M^\ast x,y\rangle_{L^2(\rho)}$ for any $x,y\in L^2(\rho)$. It is given by
the matrix $M^\ast=\drho^{-1} \transp{M}\drho$.

Working in $L^2(\rho)$ rather than the Euclidean metric on rewards and
$Q$-functions is mathematically the most natural way to have results that
still make sense for general state spaces beyond the finite case. It also
links all metrics to the data: since the distribution of samples in the
training set is approximately $\rho$, the norm
$\norm{f}^2_{L^2(\rho)}=\E_{(s,a)\sim \rho} [f(s,a)^2]$ can be estimated
empirically. In contrast, the Euclidean norm $\norm{f}^2$ would be
estimated by sampling random states uniformly distributed in the full
space (or in a given domain of $\R^n$ for continuous states): such states
might be irrelevant or even non-realistic. (See also the discussion of
reward models in Section~\ref{sec:rmodels}.)

\paragraph{Laplacian operator $\Delta$.}
For $\gamma\leq 1$, we define the \emph{Laplacian operator} of $\pi_0$ as
\begin{equation}
\Delta\deq \Id-\gamma P_{\pi_0}.
\end{equation}
The Bellman equation rewrites $\Delta Q^{\pi_0}_r=r$.

For $\gamma<1$ we have $\Delta^{-1}=\sum_{t\geq 0}\gamma^t P_{\pi_0}^t$.
For $\gamma=1$ (the most standard definition of the Laplacian), $\Delta=\Id-P_{\pi_0}$ is not invertible, as $1$ is an eigenvalue of
$P_{\pi_0}$ associated with the constant eigenvector $\1$. However, since
$\pi_0$ is ergodic, the multiplicity of this eigenvalue is $1$: by the
Perron--Frobenius theorem, $\Delta$ is invertible on the orthogonal of
the constant functions.

We denote by $L^2_0(\rho)$ the subset of functions $f$ on $S\times A$
whose average under $\rho$ is $0$. Thus, for $\gamma=1$, $\Delta$ is
invertible on $L^2_0(\rho)$.

\subsection{Regularized Policies, Regularized Return}

We are only going to consider policies
that do not deviate too much from the reference policy $\pi_0$. This is
often considered in the offline reinforcement learning setup or to
enforce safety \cite{levine2020offline}.

This can be done by adding a Kullback--Leibler
regularization term to the reward function, as described, for instance,
in \cite[\S 4.3]{levine2020offline} (policy penalty methods), which we
follow here.

\begin{defi}[ (Regularized return)]
\label{def:regul}
Let $T\geq 0$ be 
a temperature parameter. We define the \emph{regularized reward function}
for a policy $\pi$ as
\begin{equation}
\label{eq:regulr}
\bar r(s,a)\deq r(s,a) - T\, \KL{\pi(s)}{\pi_0(s)}
\end{equation}
where $\KL{\pi(s)}{\pi_0(s)}$ is the Kullback--Leibler divergence between
the policies $\pi_0$ and $\pi$ at $s$.

We define the \emph{regularized return} of policy $\pi$ for reward $r$ as
its expected return for the regularized reward, namely,
\begin{equation}
\label{eq:G}
G^\pi_r\deq \E_{s_0\sim \rho} \left[\sum_{t\geq 0} \gamma^t \bar r(s_t,a_t)
\mid s_0,\pi
\right]
\end{equation}
and we say that a policy is \emph{regularized-optimal} for $r$ if it maximizes
$G^{\pi}_r$.
\end{defi}

Maximizing $G^\pi_r$ is equivalent to maximizing the expected
return $\E_{s\sim \rho} V^{\pi_r}(s)$ plus a behavior cloning term
that keeps $\pi$ close to $\pi_0$ at each state.
% Note that $\KL{\pi_0(s)}{\pi(s)}=-\E_{a\sim \pi_0} \ln
% \pi(a|s)-\Ent(\pi_0)$ where $\Ent(\pi_0)$ does not depend on $\pi$. So
% the penalty term in the definition is $\frac{T}{1-\gamma} \E_{(s,a)\sim
% \rho} \ln \pi(a|s)$ up to a constant.

For large $T$, the maximizer $\pi$ is $O(1/T)$-close to
$\pi_0$, as the penalty term dominates. This justifies that, in the
following, we consider policies 
that are parameterized as $\Bol(f)$ for some $f$, at temperature $T$.

There are several variants of this definition: e.g., we could have
considered the KL divergence in the opposite direction, or we could have
estimated the KL term at states $s\sim \rho$ visited by $\pi_0$, instead of states $s$
visited by $\pi$ as in \eqref{eq:G}. In the regime we consider (large $T$), these variants
all lead to the same conclusions, because the differences are $O(1/T^2)$.
\option{ add proposition: same 1st order
Taylor expansion} For instance, using a behavior cloning penalty
\TODO{REF}
\begin{equation}
\E_{s\sim \rho,a\sim \pi_0(s)} \ln \pi(a|s)
\end{equation}
would provide the same conclusions.
For large $T$, maximizing the BC-regularized return
is also equivalent to performing one step of natural policy gradient starting
at $\pi_0$, with learning rate $1/T$\option{ (Section~\ref{sec:proofs},
Proposition~\ref{prop:BCisnatgrad})}. Thus, for large $T$, the
regularized return has many different interpretations.

One can check that Boltzmann policies given by the $Q$-functions of
$\pi_0$ are approximately regularized-optimal
(Corollary~\ref{enonce:optpol}). But we will need a finer result, which
describes the optimality gap depending on the policy, as follows.

% \begin{prop}
% \label{enonce:optpol}
% For any reward $r$, the policy $\pi=\Bol(Q^{\pi_0}_r)$ maximizes
% $G^\pi_r$ up to an $O(1/T^2)$ error when $T\to \infty$.
% \end{prop}

\option{ move the theorem later and just keep the proposition here}

\begin{thm}[ (Regularized return of Boltzmann policies)]
\label{thm:return}
Let $r$ be any reward function, and let $Q^{\pi_0}_r$
be the $Q$-function of the reference policy for reward $r$.
Let $\hat Q$ be any function on $S\times A$, and consider the policy
$\pi=\Bol(\hat Q)$. 

When $T\to \infty$, the regularized return of policy $\pi$ satisfies
\begin{equation}
G^{\pi}_r=G^{\pi_0}_r+\frac{1}{2T(1-\gamma)}
\left(\norm{Q^{\pi_0}_r}^2_A-\norm{\hat Q-Q^{\pi_0}_r}^2_A\right) +O(1/T^2).
\end{equation}
\option{expression for Q-func.}
\end{thm}

In particular, at first order,
this regularized return is maximal when $\hat Q=Q^{\pi_0}_r$,
corresponding to the Boltzmann policy 
$\pi=\Bol(Q^{\pi_0}_r)$. (Note that Boltzmann policies have been defined
with respect to $\pi_0$, so $\hat Q=0$ corresponds to $\pi=\pi_0$.)

\begin{cor}
\label{enonce:optpol}
For any reward $r$, the policy $\pi=\Bol(Q^{\pi_0}_r)$ maximizes
$G^\pi_r$ up to an $O(1/T^2)$ error when $T\to \infty$.
\end{cor}

Policy gradient algorithms such as PPO and TRPO also learn a new policy
$\pi_{t+1}$ by maximizing expected return subject to a constraint on the
KL divergence between $\pi_{t+1}$ and the current policy $\pi_t$
\cite{ppo,trpo}. The
KL constraint is implemented in a different way, not directly as a
penalty as in \eqref{eq:regulr}, but the intuition is similar, with 
Boltzmann policies $\pi_0\exp(\hat Q/T)$
corresponding to updating the log-probabilities of the policy $\pi_t$,
with learning rate $1/T$.

Therefore, qualitatively, Theorem~\ref{thm:return} stresses that the
optimality gap in KL-regularized policy gradient updates is directly given by
the $L^2$ error on the advantage function. \option{make a formal
statement with natural policy gradient}

\subsection{Successor Features for Zero-Shot RL, Regularized Successor
Features}

% \begin{defi}[ (Universal successor features)]
% Universal successor features (USFs) is the following procedure for
% zero-shot RL.
% 
% Let $\phi\from S\times A\to
% \R^d$ be a fixed $d$-dimensional feature map. At train time, compute a set
% of policies $(\pi_z)_{z\in \R^d}$ and a successor feature map $\psi\from
% S\times A\times \R^d \to \R^d$ satisfying simultaneously
% \begin{equation}
% \begin{cases}
% \psi(s_0,a_0,z)=\E\left[
% \sum_{t\geq 0} \gamma^t \phi(s_t,a_t) \mid s_0,a_0,\pi_z
% \right]
% \\
% \pi_z(s)=\argmax_a \transp{\psi(s,a,z)}z.
% \end{cases}
% \end{equation}
% Also denote $C\deq \E_{(s,a)\sim \rho} [\phi(s,a)\transp{\phi(s,a)}]$.
% 
% Then, at test time, given any reward function $r$, estimate 
% \begin{equation}
% z=C^{-1}
%  \E_{(s,a)\sim \rho} [r(s,a)\phi(s,a)]
% \end{equation}
%  and apply policy $\pi_z$.
% 
% If $r$ lies in the linear span of the
% features $\phi$, then the policy $\pi_z$ is optimal.
% \end{defi}
% 
% At test time, $z$ can also be estimated using a distribution other than
% $\rho$: this is useful for online situations, and preserves the optimality
% property for rewards in the linear span of $\phi$.
% 
% \bigskip

Successor features pre-compute the $Q$-functions of a number of basic reward
functions $\phi_1,\ldots,\phi_d$ or their linear combinations. At test
time, the reward function for the test task is projected onto this basis,
and a precomputed policy is applied
\cite{barreto2017successor,borsa2018universal}. This results in a
\emph{zero-shot} approach to new RL tasks in a given environment, since
no learning or planning is needed at test time. The only computation at
test time is the coefficients $z$ of the linear projection of the reward
$r$ onto the features $\phi$,
\begin{equation}
z=(\E_{(s,a)\sim \rho} [\phi(s,a)\transp{\phi(s,a)}])^{-1}
\E_{(s,a)\sim \rho}[r(s,a)\phi(s,a)]
\end{equation}
which only requires to estimate the correlation $\E_{(s,a)\sim
\rho}[r(s,a)\phi(s,a)]$ between the reward and the features. (The
covariance matrix can be precomputed.) This correlation can be estimated
empirically given some reward samples.

\bigskip

In this text, since we are concerned with \emph{regularized} policies and
return, we will only need the simplest version of successor features, in
which we only compute successor features with respect to the reference policy
$\pi_0$. This plays out as follows.

\begin{defi}[ (Regularized successor features)]
\label{def:rsf}
Regularized successor features (RSFs) is the following procedure for
regularized zero-shot RL.

Let $\phi\from S\times A\to
\R^d$ be a fixed $d$-dimensional feature map. At train time, compute a successor feature map $\psi\from
S\times A\to \R^d$ satisfying
\begin{equation}
\psi(s_0,a_0)=\E\left[
\sum_{t\geq 0} \gamma^t \phi(s_t,a_t) \mid s_0,a_0,\pi_0
\right],
\end{equation}
namely, $\psi$ solves the vector-valued Bellman equation
$\psi=\phi+\gamma P_{\pi_0} \phi$.
Also denote $C\deq \transp{\phi}\drho\phi=
\E_{(s,a)\sim \rho}
[\phi(s,a)\transp{\phi(s,a)}]$.

Then, at test time, given any reward function $r$, estimate 
\begin{equation}
\label{eq:zrsf}
z=C^{-1}
 \E_{(s,a)\sim \rho} [r(s,a)\phi(s,a)],
\end{equation}
estimate the $Q$-function of $\pi_0$ for $r$ by 
\begin{equation}
\label{eq:Qhatrsf}
\hat Q(s,a)\deq
\transp{z}\psi(s,a)
\end{equation}
and apply the Boltzmann policy
 $\hat \pi=\Bol(\hat Q)$.
\end{defi}

The following proposition is an immediate consequence of
Corollary~\ref{enonce:optpol}.

\begin{prop}
If $r$ lies in the linear span of the features $\phi$,  then
the policy $\hat \pi$ is regularized-optimal for $r$, up to an error $O(1/T^2)$.
\end{prop}

The $Q$-functions and policies recovered by SFs only depend on the linear
span of the features $\phi$. Thus, without loss of generality, we can
assume that the features are linearly independent and apply a
change of basis $\phi\gets C^{-1/2} \phi$ in feature space, after which
$C=\Id$. Thus, in the following, we always assume $C=\Id$.

Although derived in a different way, the forward-backward (FB) setup
from \cite{allpolicies,zeroshot} can also be seen as projecting the rewards onto trained
features $\phi(s,a)=(\Cov_\rho B)^{-1} B(s,a)$ at test time, where
$B$ are the features learned by the FB model \cite{zeroshot}. Therefore, our conclusions also cover this
case.

\section{Three Reward Models for Downstream Tasks}
\label{sec:rmodels}

We will compute the average performance of successor features for three
families of rewards: random Gaussian rewards (with a white noise
continuous limit), random goal-reaching (Dirac
rewards), and ``scattered random rewards''. We define each of those in
turn.

All the models depend on the distribution $\rho$ of states in the
training data (the stationary distribution of the exploration policy
$\pi_0$): either via the norm $L^2(\rho)$, or via putting rewards at
random states sampled from $\rho$. In practice, both can be estimated by
sampling random states from the
training data. \footnote{We could also have used uniform measures on
finite state spaces, or the Lebesgue measure on continuous states. But, first,
this does not extend to an abstract state space equipped with a policy
$\pi_0$, for which the $L^2(\rho)$ norm is mathematically the most natural
and the only norm available. Second, in practice, it is much more natural
to sample states from the training data than to sample random states from the
Lebesgue measure in a large domain, which might produce irrelevant or
irrealistic states, while $\rho$ will be supported on realistic states.
}

These are some of the most agnostic models we can find on
an arbitrary state equipped with an arbitrary probability distribution.
These models do not favor spatial smoothness a priori: white noise is
non-smooth and scale-free, while the goal-reaching and scattered rewards
are sparse. In these models, the Fourier transform of the reward is
uniformly spread over all frequencies. It is interesting that we can
reach meaningful conclusions about optimal features even with such
uninformative priors.

All models are built to have well-defined continuous-space limits, and
still make sense in an abstract state space equipped with a measure
$\rho$. To avoid excessive technicality, we restrict ourselves to
the finite case in this text.

\paragraph{Gaussian rewards (white noise).} For this model, we simply
sample a random Gaussian reward vector $r$ of size $S\times A$, with density 
$\propto
\exp(-\norm{r}^2_{L^2(\rho)}/2)$.
Including a variance $\sigma^2$ just rescales the rewards, so we take
$\sigma=1$ for simplicity.

The corresponding continuous-space limit is a \emph{white noise} random
reward: a random distribution $r(s,a)$ such that for any subset $X\subset
S\times A$,
the integral $\int_X r(s,a) \rho(\d s,\d a)$ is a centered Gaussian with
variance $\rho(X)$, and the integrals on two disjoint subsets $X$ and
$X'$ are independent.

This model naturally places more variance (uncertainty) on the reward at
places less covered by the training set, since the weight from
$\norm{r}^2_{L^2(\rho)}$ will be smaller where $\rho$ is small.

\paragraph{Goal-reaching (Dirac rewards).}  In this model, we first
select a random state-action $(s^\star,a^\star)\sim \rho$ in $S\times A$. Then we put
a reward $1/\rho(s^\star,a^\star)$ at $(s^\star,a^\star)$, and $0$
everywhere else:
\begin{equation}
r(s,a)=\frac{1}{\rho(s^\star,a^\star)} \,\1_{(s,a)=(s^\star,a^\star)}.
\end{equation}

The $1/\rho$ factor maintains $\int r\d \rho=1$. Without this scaling,
all $Q$-functions degenerate to $0$ in continuous spaces, 
as discussed in
\cite{blier2021unbiased}. Indeed, if we omit this factor, and just set the reward to be $1$ at
a given goal state $s^\star\in S$ in a continuous space $S$, the probability of
exactly reaching that state with a stochastic policy is usually $0$, and all
$Q$-functions are $0$. Thanks to the $1/\rho$ factor, the
continuous limit is a \emph{Dirac function} reward, infinitely sparse,
corresponding to the limit of putting a reward $1$ in a small ball
$B(s^\star,\eps)$ of radius $\eps\to 0$ around $s^\star$, and rescaling by
$1/\rho(B(s^\star,\eps))$ to keep $\int r\d \rho=1$. This produces meaningful,
nonzero $Q$-functions in the continuous limit \cite{blier2021unbiased}.

This model combines well with successor features or the FB framework:
indeed, the task representation vector $z$ in \eqref{eq:zrsf} can be computed via the
expectation 
\begin{equation}
\E_{(s,a)\sim \rho} [r(s,a)\phi(s,a)]=\phi(s^\star,a^\star)
\end{equation}
(both in finite spaces and in the continuous-space limit).

\paragraph{Scattered random rewards.} This model amounts to putting Dirac
rewards at several states instead of one, each with a different random
magnitude and sign. This describes general mixtures of sparse rewards at
several random locations.

This model depends on an intensity parameter $\kappa>0$ which
controls the number of states that have a nonzero reward. We also fix some
arbitrary distribution $p_w$ over $\R$ with some mean
$\mu$ and variance $\sigma^2\geq 0$, e.g., a Gaussian.

To build a scattered random reward, we sample a random
$N$-tuple of state-actions $((s_1,a_1),\ldots,(s_N,a_N))$ by a general
Poisson point process on $S\times A$ with intensity
$\kappa \rho$: namely, we first sample an integer $N\sim
\mathrm{Poisson}(\kappa)$, then sample $N$ random state-actions
$(s_i,a_i)\sim \rho$, $i=1,\ldots, N$, independently.
\footnote{
The Poisson law ensures that the number of states selected in some part
of $S\times A$ is independent from the number of states selected in any other
part.
In this process, the number of state-actions
selected in any part $X\subset S\times A$ follows a Poisson law with
parameter $\kappa \rho(X)$, with distinct subsets being independent.
Once more, this ensures a meaningful continuous-time limit,
as the general Poisson process is well-defined in a continuous space with
measure $\rho$.}
At each $(s_i,a_i)$, we place a
Dirac reward as above, but multiply it by some
random weight $w_i\sim p_w$, sampled independently from everything else. The reward is $0$
on the rest of the space. Explicitly,
\begin{equation}
r(s,a)=\sum_{i=1}^N \frac{w_i}{\rho(s_i,a_i)}
\1_{(s,a)=(s_i,a_i)}.
\end{equation}
As for Dirac rewards above, the factor $1/\rho$ ensures a
meaningful limit for continuous spaces (where $r$ becomes a random
distribution, a random sum of Dirac functions wrt $\rho$). In this model,
the SF task representation vector $z$ in \eqref{eq:zrsf} is given by
\begin{equation}
z=C^{-1}\sum_{i=1}^N w_i\, \phi(s_i,a_i)
\end{equation}
both in finite spaces and in the continuous-space limit.

\section{The Bellman Gap Norm is Uninformative for Successor Features}

We start with a negative result: for the families of rewards above, the
size of expected Bellman gaps of the $Q$-functions estimated by SFs is
\emph{independent} of the choice of features: it only depends on the
number of linearly independent features.

This \emph{does not mean that all choices of features perform equally
well}: as we will see, regularized returns do depend on the features. It
just means that the squared Bellman gap error is a poor proxy.

\begin{prop}[ (Average Bellman gaps do not depend on the features)]
\label{prop:gap}
Given a reward $r$, let $\hat Q(s,a)\deq \transp{z}\psi(s,a)$ be the
$Q$-function \eqref{eq:Qhatrsf} estimated by regularized successor features, with $z$ given
by \eqref{eq:zrsf}. Assume the features are linearly independent.

Then, for either the random Gaussian reward or the random goal-reaching reward of Section~\ref{sec:rmodels}, on
average for $r$ in the model,
the norm of the Bellman gaps of $\hat Q$ only depends on the
number of features $d$. More precisely,
\begin{equation}
\E_r \norm{\hat Q-r-\gamma P_{\pi_0} \hat Q}^2_{L^2(\rho)}=\#S\times \#A-d
\end{equation}
where the expectation is over a random reward $r$ from the model.
\option{ scattered model}
\end{prop}

So, even if we had access to downstream reward functions $r$, we could
not learn good features by minimizing the expected Bellman gaps on those
rewards.  No meaningful analysis of the performance of a set of features
is possible based on the norm of Bellman errors.

We have expressed this theorem using
Bellman gaps with respect to the reference policy $\pi_0$: this is what is needed for
regularized SFs, thanks to Theorem~\ref{thm:return} and
Corollary~\ref{enonce:optpol}. However, a
similar
result holds for \emph{universal} successor features
\cite{borsa2018universal}, with Bellman gaps
taken for the estimated optimal policy $\hat \pi$ for each reward
$r$: the proof in Section~\ref{sec:proofs} directly covers this case
as well.

\bigskip

It might be surprising that the expected norm of Bellman gaps does not
align well with the optimal return, since in practice, $Q$-functions are
routinely learned by the TD algorithm based on reducing Bellman gaps. But
TD does not actually minimize the expected norm of Bellman gaps, due to
the ``double sampling problem'' and updating only the $Q$ in the
left-hand-side of the Bellman equation. In general, TD does not minimize
any norm and can diverge \cite{tsitsiklis1997approxtd}, but in certain cases it is known to
minimize the \emph{Dirichlet norm} \footnote{The Dirichlet (semi)norm of
a function $f$ is $\langle f,(\Id -\gamma
P_{\pi})f\rangle_{L^2(\rho)}\geq 0$, also equal to $\frac12\E
(f(s_{t+1},a_{t+1})-f(s_t,a_t))^2)$. We refer to \cite{tdconv}. This is also
the objective minimized in Laplacian eigenfunctions.} of the error on $Q$
\cite{tdconv}. For
$\gamma$ close to $1$, the Dirichlet norm is the sum of the Bellman gap
norm and the advantage seminorm (Proposition~\ref{prop:Anorm} with
$\gamma=1$\option{make that a separate statement}). The advantage seminorm
is the one that matters for KL-regularized policy gradient
(Theorem~\ref{thm:return}). So in
the end, the discrepancy between TD and minimizing expected Bellman gaps
might be a blessing.

\todo{explain: expected BG norm is just reconstruction error on $r$. This
completely ignores the dynamics, and also includes some irrelevant stuff
eg, adding constants to rewards. }

\section{The Optimal Features for Regularized Successor Features}

We first introduce the \emph{advantage kernel}, a symmetric matrix that
describes the norm of the advantage function for a given reward function.

\begin{defi}
\label{def:AK}
Let $\pi$ be a fixed policy with invariant distribution $\rho$, and consider the map that to any reward
function $r$ associates
the norm $\norm{A^{\pi}_r}^2_{L^2(\rho)}$ of its advantage function.
Since the $Q$-function is linear in $r$ for a given $\pi$, this is a
quadratic function of $r$ for fixed $\pi$.

Therefore, there exists a symmetric positive semi-definite matrix\footnote{or
positive semi-definite kernel
for infinite state spaces $S$} $\AK_\pi$ of size $(\#S\times \#A)\times
(\#S\times \#A)$ such that
\begin{equation}
\norm{A^{\pi}_r}^2_{L^2(\rho)}=\transp{r} \AK_\pi r
\end{equation}
for any $r$. We call this matrix the \emph{advantage kernel} of policy
$\pi$.
\end{defi}

Since the $Q$-function of reward $r$ for policy $\pi_0$ is
$Q^{\pi_0}_r=\Delta^{-1}r$, one can compute $\AK_{\pi_0}$ 
in terms of $\Delta^{-1}$. Expressing the norm of the
advantage function as the difference between the norms of the
$Q$-function and value function, 
$\norm{Q^{\pi_0}_r}^2_A=\norm{Q^{\pi_0}_r}^2_{L^2(\rho)}-\norm{\pi_0
Q^{\pi_0}_r}^2_{L^2(\rho_S)}=\transp{(Q^{\pi_0}_r)}(\drho-\transp{\pi_0}\drho_S\pi_0)Q^{\pi_0}_r$,
we obtain
\begin{equation}
\AK_{\pi_0}=\transp{(\Delta^{-1})} (\drho
-\transp{\pi_0}\drho_S\pi_0)\Delta^{-1}.
\end{equation}

\option{ should we use $\langle r, \AK_\pi r\rangle_{L^2(\rho)}$
everywhere
instead?}

\bigskip

The following theorem expresses the expected regularized return for
regularized successor features depending on the features, for each of the
three reward models defined in Section~\ref{sec:rmodels}, and up to an
$O(1/T^2)$ error. The optimal
choice of features corresponds to the features that maximize
$\transp{x}\AK_{\pi_0}x$ for a given norm $\norm{x}^2_{L^2(\rho)}$.

\begin{thm}[ (Expected regularized return depending on the features)]
\label{thm:featurereturn}
Let $\phi$ be a set of basic features for successor features. Without
loss of generality, we
assume that these features are $L^2(\rho)$-orthonormal.

Given a reward $r$, let $\hat Q(s,a)\deq \transp{z}\psi(s,a)$ be the
$Q$-function \eqref{eq:Qhatrsf} estimated by regularized successor features, with $z$ given
by \eqref{eq:zrsf}. Let $\hat \pi=\Bol(\hat Q)$ be the estimated
regularized-optimal policy.

Then, for the reward models of Section~\ref{sec:rmodels}, on
average for $r$ in the model,
the regularized return $G^{\hat \pi}_r$ of the estimated policy
satisfies:
\begin{itemize}
\item
For the random Gaussian or random goal-reaching reward models,
\begin{equation}
\E_r [G^{\hat \pi}_r]=\E_r [G^{\pi_0}_r] + \frac{1}{T(1-\gamma)}
\Tr(\transp{\phi}\AK_{\pi_0}\phi) + O(1/T^2).
\end{equation}
\item
For the scattered random reward model with intensity $\kappa$, mean $\mu$
and variance $\sigma^2$,
\begin{equation}
\label{eq:scatteredgain}
\E_r [G^{\hat \pi}_r]=\E_r [G^{\pi_0}_r] +
\frac{1}{T(1-\gamma)}\left(
\kappa (\mu^2+\sigma^2)\Tr(\transp{\phi}\AK_{\pi_0}\phi)
-(\kappa\mu)^2 \transp{\phi_\mathrm{cst}}\,\AK_{\pi_0}\phi_\mathrm{cst}
\right)
+ O(1/T^2)
\end{equation}
where $\phi_\mathrm{cst}$ is the $L^2(\rho)$-orthogonal projection of the
constant reward $r=\1$ onto the span of the features.

\end{itemize}
%where $\Pi$ is the $L^2(\rho)$-orthogonal projector onto the features
%$\phi$ used in successor features.
\end{thm}

This result allows us to work out which features optimize the gain: those
that maximise $\Tr(\transp{\phi}\AK_{\pi_0}\phi)$ under the constraint
that $\phi$ is $L^2(\rho)$-orthonormal, leading to the following
characterization.  Scattered random rewards
require a slightly longer proof to handle the extra term in
\eqref{eq:scatteredgain}, but the conclusion is the same.

\begin{cor}[ (Optimal features for regularized successor features)]
\label{cor:optfeatures}
For any of the three reward models of Section~\ref{sec:rmodels}, the
features $\phi$ that bring maximal regularized return up to $O(1/T^2)$ are
the largest $d$ eigenvectors of $\drho^{-1}\AK_{\pi_0}$, or equivalently
the
largest $d$ extremal directions of
$\transp{x}\AK_{\pi_0}x/\norm{x}^2_{L^2(\rho)}$.
In that case we have
\begin{equation}
\Tr(\transp{\phi}\AK_{\pi_0} \phi)=\sum_{i=1}^d \lambda_i
\end{equation}
where $\lambda_1,\ldots,\lambda_d$ are the largest $d$ eigenvalues of
$\drho^{-1}\AK_{\pi_0}$. \footnote{The operator $\drho^{-1}\AK_{\pi_0}$ is
self-adjoint in $L^2(\rho)$, so it is diagonalizable. Its eigenvalues are
the same as those of the symmetric matrix $\drho^{-1/2}\AK_{\pi_0}\drho^{-1/2}$.}
\end{cor}

We now turn to a more explicit characterization of these eigenvectors 
in deterministic environments. The results depend on the
decay factor $\gamma$. We first consider the two extreme cases $\gamma=1$
and $\gamma=0$, as they lead to the simplest expressions.

Remember that $\AK_{\pi_0}$ vanishes for a constant reward (the advantage
function is $0$), so we only have to compute it for rewards orthogonal to
the constants, $r\in L^2_0(\rho)$. For such rewards, $Q$-functions and
advantage functions are defined even for $\gamma=1$, and the Laplacian
operator is invertible up to $\gamma=1$.

\begin{thm}[ (Optimal features for $\gamma=1$ in a
deterministic environment)]
\label{thm:lap}
Assume the environment is deterministic.
For $\gamma=1$ and $r\in L_0^2(\rho)$, the advantage kernel is given by
\begin{equation}
\transp{r}\AK_{\pi_0}r
%=\transp{r} \left( \drho\Delta^{-1}+ \transp{(\drho \Delta^{-1})}-\drho \right)r
=\langle r, (\Delta^{-1}+(\Delta^{-1})^\ast-\Id)r\rangle_{L^2(\rho)}
\end{equation}
where $\Delta\deq \Id-P_{\pi_0}$ is the Laplacian operator of $\pi_0$,
and where $(\Delta^{-1})^\ast=\drho^{-1}\transp{(\Delta^{-1})}\drho$ is the adjoint of $\Delta^{-1}$ acting on
$L^2_0(\rho)$.

Consequently, for $\gamma=1$ in a deterministic environment, the optimal features are the largest
eigenfunctions of $\Delta^{-1}+(\Delta^{-1})^\ast$.
\end{thm}

Intuitively, the largest eigenfunctions of
$\Delta^{-1}+(\Delta^{-1})^\ast$ correspond to the lowest frequencies
(longest range variations) in the environment: for $\gamma=1$, we want to keep
information on the large-scale variations of the reward function.

For comparison, it has previously been suggested to use as features the smallest
eigenfunctions of $\Delta+\Delta^\ast$ \cite{wu2018laplacian}, or equivalently, the largest
eigenfunctions of $P_{\pi_0}+P_{\pi_0}^\ast$.
The largest
eigenfunctions of $P_{\pi_0}+P_{\pi_0}^\ast$ also convey a low-frequency
intuition,
but in a somewhat different way. Indeed, in general, symmetrizing does
not commute with taking the inverse: $\Delta^{-1}+(\Delta^{-1})^\ast\neq
(\Delta+\Delta^\ast)^{-1}$ in general. This equality can still happen,
for instance if $\Delta$ is reversible, but reversibility of $\Delta$ is
a very specific situation: for instance, this is never the case in
kinematic environments. \footnote{Indeed, reversibility implies that if a
transition $(s,a)\to (s',a')$ is possible, then so is the reverse
transition $(s',a')\to (s,a)$. This is not the case if speed is part of
the state: if $s=(x,v)$ then the next state has $x'\approx x+\delta t\, v$, so the transition
$(s',a')\to (s,a)$ would require $v'\approx -v$, namely, the ability to
fully reverse speed with a single action.

Note that here we deal with
reversibility in the language of Markov chain theory, not reversibility in the
language of physics: in a physicist's language, classical mechanics are
reversible (by changing $v$ to $-v$).}

Therefore, in general, the optimal features are \emph{not} the smallest
eigenfunctions of $\Delta+\Delta^\ast$.

\bigskip

In contrast, for small $\gamma$, what matters are the \emph{highest}
frequencies of the reward function, as we show now.

\begin{prop}[ (Optimal features for $\gamma=0$ in a deterministic
environment)]
\label{prop:gamma0}
Assume the environment is deterministic.
For $\gamma=0$ and $r\in L^2(\rho)$, the advantage kernel is given by
\begin{equation}
\transp{r}\AK_{\pi_0}r
%=\transp{r} \left( \drho\Delta^{-1}+ \transp{(\drho \Delta^{-1})}-\drho \right)r
=\langle r, (\Id-P_{\pi_0}^\ast P_{\pi_0})r\rangle_{L^2(\rho)}
\end{equation}
where $P_{\pi_0}^\ast=\drho^{-1}\transp{P_{\pi_0}}\,\drho$ is the adjoint
of $P_{\pi_0}$ acting on
$L^2(\rho)$.

Consequently, for $\gamma=0$ in a deterministic environment, the optimal features are the
\emph{smallest}
eigenfunctions of $P_{\pi_0}^\ast P_{\pi_0}$.
\end{prop}

For $\gamma=0$, the problem is a bandit problem: given an initial state
$s_0$, just pick the action with the highest reward, after which the game
ends. This relies on knowing how to compare the value of $r$ on actions
at the same state, which is high-frequency information about $r$.
\footnote{This conclusion for $\gamma=0$ heavily depends on the fact that we have
defined rewards over \emph{state-actions} all along. If rewards only
depend on the state, there is no meaningful optimal behavior for
$\gamma=0$. One could study the $\gamma\to 0$ limit for reward models
depending only on the state, but this is beyond the scope of the present
work.}

\bigskip

Finally, we give the general expression for $0<\gamma<1$ in a
deterministic environment.

\begin{thm}[ (Optimal features for $0<\gamma<1$ in a deterministic
environment)]
\label{thm:gen}
Assume the environment is deterministic.
For $0<\gamma<1$, the advantage kernel is given by
\begin{align*}
\transp{r}\AK_{\pi_0}r=
\frac1{\gamma^2}\left\langle r, \left(\Delta^{-1}+(\Delta^{-1})^{\ast}-\Id-(1-\gamma^2)
(\Delta^{-1})^\ast \Delta^{-1}\right)r\right\rangle_{L^2(\rho)}
\end{align*}
where $\Delta\deq \Id-\gamma P_{\pi_0}$ is the Laplacian operator of $\pi_0$,
and where $(\Delta^{-1})^\ast=\drho^{-1}\transp{(\Delta^{-1})}\drho$ is the adjoint of $\Delta^{-1}$ acting on
$L^2(\rho)$.

Consequently, in a deterministic environment, the optimal features are
the largest eigenfunctions of 
$\Delta^{-1}+(\Delta^{-1})^{\ast}-(1-\gamma^2)
(\Delta^{-1})^\ast \Delta^{-1}$.
\end{thm}

% (Beware that the optimal features are the extremal points of
% $\AK_{\pi_0}/\norm{\cdot}^2_\rho$, not those of
% $\AK_{\pi_0}/\norm{\cdot}^2$, so the optimal features are not the
% eigenfunctions of $\drho\Delta^{-1}+ \transp{(\drho \Delta^{-1})}$. This
% is why we express everything using the $L^2(\rho)$ inner product.)

\option{
\section{The Advantage Kernel and Natural Policy Gradient}

\TODO{}

Finally, the advantage kernel gives the performance improvement along natural policy gradient
trajectories (Theorem \TODO{}): if $(\pi_t)_{t\geq 0}$ is the family of
policies learned by natural policy gradient on a reward $r$, then
\begin{equation}
\frac{\d}{\d t} \E_{\rho_t} [r]=\transp{r}\AK_{\pi_t} r
\end{equation}
where $\rho_t$ is the invariant distribution of $\pi_t$.

\TODO{} move Theorem~\ref{thm:return} here?
}

%\appendix

\section{Proofs}
\label{sec:proofs}

% \paragraph{Notation for the proofs.} We view value functions as column
% vectors of size $\# S$, rewards and $Q$-functions as column vectors of size
% $(\#S\times \#A)$, the transition probabilities $P$ as a matrix of size
% $(\#S\times \#A)\times \#S$ with entries $P_{(sa)s'}=P(s'|s,a)$, and 
% policies $\pi$ as matrices of size $\#S \times (\#S\times \#A)$ with
% entries $\pi_{s(s'a)}=\pi(a|s)\1_{s'=s}$, so that $V=\pi Q$. With this notation, the Bellman
% equations are $V=\pi r+\gamma \pi P V$ and $Q=r+\gamma P\pi Q$.
% 
% \TODO{}
% 
% \TODO{}: do i need $\rho$ the invariant distribution, or just any $\rho_0$
% then $\rho=$ states visited under $\pi$ starting from $\rho_0$?

\option{
\paragraph{BC regularization with large $T$ is equivalent to natural
policy gradient.}

Nat grad = Fisher.policy grad by def.  P(theta).df/dtheta equiv to
penalty P at first order. Only first-order behavior of objective matters.
First-order objectives coincide by pol grad thm (note: additional
}

\begin{lem}
\label{lem:KLboltz}
Abbreviate $K(\pi,s)\deq T \KL{\pi(s)}{\pi_0(s)}$ for the regularization
term in Definition~\ref{def:regul}.

Then
\begin{equation}
\Bol(f)(s,a)=\piref(s,a)\left(1+\frac1T f(s,a)-\frac1T\bar f(s)\right)+O(1/T^2)
\end{equation}
and
\begin{equation}
K(\Bol(f),s)
%\KL{\Bol(f)(s)}{\piref(s)}
=\frac1{2 T} \E_{a\sim \piref(s)} \left(
f(s,a)-\bar f(s)\right)^2+O(1/T^2)
\end{equation}
and consequently
\begin{equation}
\E_{s\sim \rho} \left[%\KL{\Bol(f)(s)}{\piref(s)}
K(\Bol(f),s)
\right]
=\frac1{2T} \norm{f}^2_A+O(1/T^2)
\end{equation}
\end{lem}

\begin{dem}
These follow from direct Taylor expansions.
\end{dem}

\begin{dem}[ of Theorem~\ref{thm:return}]
We want to estimate the regularized return $G^\pi_r$ of $\pi=\Bol(\hat Q)$. By
definition, it is the sum of the ordinary return and a penalty term,
\begin{equation}
G^{\pi}_r= \E_{s_0\sim \rho} [V^{\pi}_r(s_0)] - \E_{s_0\sim\rho}
\,\E\left[
\sum_{t \geq 0} \gamma^t K(\pi,s_t)\mid s_0,\pi
\right]
\end{equation}
where $K(\pi,s)\deq T\KL{\pi(s)}{\pi_0(s)}$ is the penalty term as defined in Lemma~\ref{lem:KLboltz}.

We first estimate the value function $V^{\pi}_r$, then turn to the
penalty term.

By definition of Boltzmann policies, $\pi=\Bol(\hat Q)$ is $O(1/T)$-close to
$\pi_0$.

For $\pi$ close to $\pi_0$, the policy gradient theorem provides the
expression of the derivative of $V^{\pi}$ with respect to $\pi$. Writing
the policy gradient theorem as a Taylor expansion around $\pi\approx
\pi_0$, we obtain
\begin{equation}
V^\pi(s_0)-V^{\pi_0}(s_0)=\E_{(s_t,a_t)} \left[
\sum_{t\geq 0} \gamma^t
A^{\pi_0}(s_t,a_t)
\left(\ln \pi(a_t|s_t)-\ln \pi_0(a_t|s_t)\right)
\mid s_0,\pi_0
\right]+
O((\pi-\pi_0)^2)
\end{equation}
where $A^{\pi_0}$ is the advantage function of policy $\pi_0$.

Let us average this over $s_0\sim \rho$. Since $\rho$ is the invariant
distribution of $\pi_0$, each $s_t$ above is also distributed according
to $\rho$, and therefore all values of $t$ make the same contribution:
\begin{equation}
\E_{s_0\sim \rho} [V^\pi(s_0)-V^{\pi_0}(s_0)]
=\frac{1}{1-\gamma} \E_{s\sim \rho,a\sim \pi_0(s)}\left[
A^{\pi_0}(s,a)\left(\ln \pi(a|s)-\ln \pi_0(a|s)\right)
\right]
+
O((\pi-\pi_0)^2)
\end{equation}

By Lemma~\ref{lem:KLboltz}, for $\pi=\Bol(\hat Q)$ we have
$(\pi-\pi_0)^2=O(1/T^2)$ and moreover
\begin{equation}
\ln \pi(s,a)=\ln \pi_0(s,a)+ \frac1T(\hat Q(s,a)-\bar {\hat Q}(s))+O(1/T^2)
\end{equation}
and therefore
\begin{multline}
\E_{s_0\sim \rho} [V^\pi(s_0)-V^{\pi_0}(s_0)]
\\=\frac{1}{T(1-\gamma)} \E_{s\sim \rho,a\sim \pi_0(s)}\left[
A^{\pi_0}(s,a)\left(\hat Q(s,a)-\bar{\hat Q}(s)\right)\right]
+O(1/T^2)
\\=\frac{1}{T(1-\gamma)}\langle Q_r^{\pi_0},\hat Q\rangle_A + O(1/T^2)
\end{multline}
by definition of the advantage norm.

Let us now turn to the regularization term in $G^{\pi}_r$: we want to estimate
\begin{equation}
\E_{s_0\sim \rho} \,\E\left[
\sum_{t \geq 0} \gamma^t K(\pi,s_t)\mid s_0,\pi
\right].
\end{equation}

Let $p_{\pi,t}$ be the distribution of $s_t$ under policy $\pi$ and
$s_0\sim \rho_0$. Since $\pi=\Bol(f)$ is $O(1/T)$-close to $\pi_0$, the
associated 
transition matrices are also close: $P_{\pi}=P_{\pi_0}+O(1/T)$, and
therefore $P_\pi^t=P_{\pi_0}^t+O(1/T)$.
\footnote{The constant in $O$ is
not uniform in $t$. It grows like $t$, since
$P_1^t-P_2^t=\sum_{i=0}^{t-1} P_1^i
(P_1-P_2) P_2^{t-1-i}$, and $P_1$ and $P_2$ are stochastic matrices so are
non-expanding in sup norm, so each term is $O(P_1-P_2)$. So we have $P_1^t-P_2^t= O(t(P_1-P_2))$ with
uniform constants. The $t$ factor is absorbed by $\gamma^t$ in the
cumulated return.}
So in turn, $p_{\pi,t}=p_{\pi_0,t}+O(1/T)$.

By Lemma~\ref{lem:KLboltz}, $K$ itself is $O(1/T)$. So when computing the
expectation of $K$ under $s_t\sim p_{\pi,t}$, we can replace $\pi$ with
$\pi_0$ up to an $O(1/T^2)$ error:
\begin{align}
\E_{s_t\sim p_{\pi,t}} [K(\pi,s_t)]&=
\E_{s_t\sim p_{\pi_0,t}}
[K(\pi,s_t)]+O(\norm{K}\norm{p_{\pi,t}-p_{\pi_0,t}})
\\&=
\E_{s_t\sim p_{\pi_0,t}} [K(\pi,s_t)]+O(1/T^2).
\end{align}

Since $s_0\sim \rho$ and $\rho$ is an invariant distribution of $\pi_0$,
we have $p_{\pi_0,t}=\rho$. Therefore,
\begin{equation}
\sum_{t\geq 0} \gamma^t \E_{s_t\sim p_{\pi_0,t}}
[K(\pi,s_t)]=\frac{1}{1-\gamma}\E_{s\sim \rho} [K(\pi,s)]
\end{equation}
and since $\pi=\Bol(\hat Q)$, 
this is equal to 
$\frac1{2T(1-\gamma)} \norm{\hat Q}^2_A+O(1/T^2)$ by Lemma~\ref{lem:KLboltz}.

By collecting the return term and the penalty term, we have
\begin{align}
G^{\pi}_r-\E_{s_0\sim \rho}[V^{\pi_0}(s_0)]&=
\frac{1}{T(1-\gamma)}\langle Q_r^{\pi_0},\hat Q\rangle_A 
- \frac1{2T(1-\gamma)} \norm{\hat Q}^2_A+O(1/T^2)
\\&=
\frac{1}{2T(1-\gamma)}\left(
\norm{Q_r^{\pi_0}}^2_A-\norm{Q_r^{\pi_0}-\hat Q}^2_A
\right)
+O(1/T^2)
\end{align}
which ends the proof since $G^{\pi_0}_r=\E_{s_0\sim
\rho}[V^{\pi_0}(s_0)]$.
\end{dem}

\bigskip

For the proof of Theorem~\ref{thm:featurereturn} we need a preliminary
result on the reward models.

\begin{prop}[ (Second moment of the reward in the models)]
\label{prop:Err}
For the random reward models of Section~\ref{sec:rmodels}, the second
moment $\E[r\transp{r}]$ satisfies:
\begin{itemize}
\item For the random Gaussian reward and random goal-reaching reward,
\begin{equation}
\E [r\transp{r}]=\drho^{-1}.
\end{equation}
\item For the scattered random reward model with intensity $\kappa$, mean $\mu$
and variance $\sigma^2$,
\begin{equation}
\E [r\transp{r}]=
\kappa (\mu^2+\sigma^2)\drho^{-1}+
(\kappa\mu)^2 \1\transp{\1}
\end{equation}
where $\1$ is the constant vector with components equal to $1$.
\end{itemize}
\end{prop}

\begin{dem}[ of Proposition~\ref{prop:Err}]
For the random Gaussian reward, the model is $\propto \exp
(-\norm{r}^2_{L^2(\rho)}/2)=\exp(-\transp{r}\drho r/2)$ so the covariance
matrix is $\drho^{-1}$ by construction.

For random goal-reaching, we first sample a state-action
$(s^\star,a^\star)\sim \rho$
then set the reward to $r=\1_{(s^\star,a^\star)}/\rho(s^\star,a^\star)$. Therefore, the expectation
of $r\transp{r}$ is
\begin{equation}
\E[r\transp{r}]=\sum_{(s^\star,a^\star)}
\rho(s^\star,a^\star)\frac{\1_{(s^\star,a^\star)}\transp{\1_{(s^\star,a^\star)}}}{\rho(s^\star,a^\star)^2}=\sum_{(s^\star,a^\star)}\frac{1}{\rho(s^\star,a^\star)}\1_{(s^\star,a^\star)}\transp{\1_{(s^\star,a^\star)}}=\drho^{-1}.
\end{equation}

Scattered random reward require more computation. We first sample an
integer $N\sim \mathrm{Poisson}(\kappa)$, then sample $N$ state-actions
$(s_i,a_i)\sim\rho$ and weights $w_i\sim p_w$, $i=1,\ldots,N$, then set
the reward to
\begin{equation}
r=\sum_{i=1}^N \frac{w_i}{\rho(s_i,a_i)} \1_{(s_i,a_i)}.
\end{equation}
Therefore,
\begin{equation}
\E[r\transp{r}]=\E\left[
\sum_{i=1}^N \frac{w_i^2}{\rho(s_i,a_i)^2}
\1_{(s_i,a_i)}\transp{\1_{(s_i,a_i)}}
+\sum_{i=1}^N\frac{w_i}{\rho(s_i,a_i)} \1_{(s_i,a_i)}
\sum_{j\neq i} \frac{w_j}{\rho(s_j,a_j)} \transp{\1_{(s_j,a_j)}}
\right]
\end{equation}

Let us consider the first term. For each $i$ the expectation is the same,
and moreover the sampling of the weights $w_i$ is independent from the
sampling of the state-actions, so
\begin{align}
\E
\left[
\sum_{i=1}^N \frac{w_i^2}{\rho(s_i,a_i)^2}
\1_{(s_i,a_i)}\transp{\1_{(s_i,a_i)}}
 \right]
&=(\E[N])\,(\E[w_1^2])\,\E\left[
\frac{1}{\rho(s_1,a_1)^2}
\1_{(s_1,a_1)}\transp{\1_{(s_1,a_1)}}
\right]
\\&=\kappa (\mu^2+\sigma^2) \drho^{-1}
\end{align}
because the expectation of $N\sim \mathrm{Poisson}(\kappa)$ is $\kappa$, and
because $\E\left[
\frac{1}{\rho(s_1,a_1)^2}
\1_{(s_1,a_1)}\transp{\1_{(s_1,a_1)}}
\right]$ is computed exactly as in the random goal-reaching reward case.

For the second term, each $j$ term is independent from the $i$ term. For
each $j$ term, we have
\begin{align}
\E\left[
\frac{w_j}{\rho(s_j,a_j)} \transp{\1_{(s_j,a_j)}}
\right]
&=\E[w_j]\,\E\left[\frac{1}{\rho(s_j,a_j)} \transp{\1_{(s_j,a_j)}}\right]
\\&=\mu \sum_{(s,a)} \rho(s,a) \frac{1}{\rho(s,a)} \transp{\1_{(s,a)}}
\\&=\mu\transp{\1}
\end{align}
since the probability to sample $(s,a)$ is $\rho(s,a)$, and since the
weights $w_j$ are sampled independently from the state-actions.

The same computation applies to the $i$ term, which is independent from
the $j$ term. So conditionally to $N$, each pair $(i,j)$ contributes
$\mu^2 \1\transp{\1}$ to the expectation.
Since there are $N(N-1)$ pairs $(i,j)$, we have
\begin{equation}
\E\left[
\sum_{i=1}^N\frac{w_i}{\rho(s_i,a_i)} \1_{(s_i,a_i)}
\sum_{j\neq i} \frac{w_j}{\rho(s_j,a_j)} \transp{\1_{(s_j,a_j)}}
\right]=\E[N(N-1)] \,\mu^2 \1\transp{\1}.
\end{equation}
Finally, for a Poisson process with parameter $\kappa$, we have
$\E[N(N-1)]=\E[N^2]-\E[N]=\Var[N]+(\E[N])^2-\E[N]=\kappa^2$ since the
expectation and variance of the Poisson distribution are both $\kappa$.
This ends the proof.
\end{dem}

The following lemma expresses that constant rewards have zero advantages.

\begin{lem}
\label{lem:A1}
For any policy $\pi$, one has $\AK_\pi\1=0$ where $\1$ is the constant
vector with components $1$.
\end{lem}

\begin{proof}
By definition, for any reward function $r$, one has
\begin{equation}
\transp{r}\AK_\pi r= \norm{Q^\pi_r}^2_A= \E_{(s,a)\sim \rho}\left[
(Q^\pi_r(s,a)-\E_{a'\sim \pi_0(s)} Q^\pi_r(s,a'))^2\right]
\end{equation}
and the polar form of this quadratic form is the correlation between
advantages,
\begin{equation}
\transp{r_1}\AK_\pi r_2= \E_{(s,a)\sim \rho}[
(Q^\pi_{r_1}(s,a)-\E_{a'\sim \pi_0(s)} Q^\pi_{r_1}(s,a'))
(Q^\pi_{r_2}(s,a)-\E_{a'\sim \pi_0(s)} Q^\pi_{r_2}(s,a'))
].
\end{equation}

When $r_2=\1$, one has $Q^\pi_{r_2}=\frac{1}{1-\gamma}\1$ for any policy
$\pi$. Therefore, $Q^\pi_{r_2}(s,a)-\E_{a'\sim \pi_0(s)}
Q^\pi_{r_2}(s,a')=0$ and $\transp{r_1}\AK_\pi r_2=0$ for any $r_1$, so
that $\AK_\pi r_2=0$.
\end{proof}

\begin{dem}[ of Theorem~\ref{thm:featurereturn}]
By Theorem~\ref{thm:return}, for a reward $r$, the expected return using
policy $\Bol(\hat Q)$ is
\begin{equation}
G^{\pi}_r=G^{\pi_0}_r+\frac{1}{2T(1-\gamma)}
\left(\norm{Q^{\pi_0}_r}^2_A-\norm{\hat Q-Q^{\pi_0}_r}^2_A\right) +O(1/T^2).
\end{equation}
and we want to compute the expectation of this when the reward follows
one of the models in Section~\ref{sec:rmodels}, and $\hat Q$ is the
$Q$-function estimated by regularized successor features.

To obtain
Theorem~\ref{thm:featurereturn}, we must compute the expectation over $r$
of the $\frac{1}{2T(1-\gamma)}$ term.

Let $\hat r$ be the $L^2(\rho)$-orthogonal projection of the reward $r$
onto the features $\phi$. By Definition~\ref{def:rsf}, successor features estimate
$\hat Q$ as the $Q$-function of the estimated reward $\hat r$ for policy
$\pi_0$, namely,
\begin{equation}
\hat Q=Q^{\pi_0}_{\hat r}
\end{equation}
and therefore
\begin{equation}
\hat Q-Q^{\pi_0}_r=Q^{\pi_0}_{\hat r}-Q^{\pi_0}_r=Q^{\pi_0}_{\hat r-r}
\end{equation}
since $Q$-functions are linear in the reward for a fixed policy.

By definition of the advantage kernel $\AK$, the advantage norm of the
$Q$-function of reward $r$ is $\transp{r}\AK r$, so
\begin{equation}
\norm{Q^{\pi_0}_r}^2_A=\transp{r} \AK_{\pi_0}r=\Tr
(\AK_{\pi_0}r\transp{r})
\end{equation}
and
\begin{equation}
\norm{\hat Q-Q^{\pi_0}_r}^2_A=\norm{Q^{\pi_0}_{\hat r-r}}^2_A=
\transp{(r-\hat r)}\AK_{\pi_0} (r-\hat r)
=\Tr(\AK_{\pi_0}(r-\hat r)\transp{(r-\hat r)}).
\end{equation}
By definition of $\hat r$, we have $r-\hat r=(\Id-\Pi)r$ where $\Pi$ is
the $L^2(\rho)$-orthogonal projector onto the features. So
\begin{multline}
\label{eq:return_Err}
\E_r\left[\norm{Q^{\pi_0}_r}^2_A-\norm{\hat Q-Q^{\pi_0}_r}^2_A\right]=
\Tr(\AK_{\pi_0} \E_r [r\transp{r}])
\\-\Tr\left(\AK_{\pi_0} (\Id-\Pi)\E_r
[r\transp{r}] \transp{(\Id-\Pi)}\right)
\end{multline}

Since the
features $\phi$ are $L^2(\rho)$-orthonormal, this projector is given by
$\Pi r=\phi w$ where $w=\E_{(s,a)\sim \rho}
[r(s,a)\phi(s,a)]=\transp{\phi}\drho r$ is the linear regression vector
of $r$ onto the features. Therefore the projector $\Pi$ satisfies $\Pi
r=\phi\transp{\phi}\drho r$ so
\begin{equation}
\Pi=\phi\transp{\phi}\drho.
\end{equation}

Let us first consider the case of random Gaussian rewards and random
goal-reaching rewards. In both cases, by Proposition~\ref{prop:Err} we
have
\begin{equation}
\E[r\transp{r}]=\drho^{-1}
\end{equation}
so in that case
\begin{multline}
\E_r\left[\norm{Q^{\pi_0}_r}^2_A-\norm{\hat Q-Q^{\pi_0}_r}^2_A\right]=
\Tr(\AK_{\pi_0} \drho^{-1})
\\-\Tr\left(\AK_{\pi_0} (\Id-\Pi)\drho^{-1}\transp{(\Id-\Pi)}\right).
\end{multline}
Since $\Id-\Pi$ is an $L^2(\rho)$-orthogonal projector, one has
\begin{equation}
(\Id-\Pi)\drho^{-1}\transp{(\Id-\Pi)}=(\Id-\Pi)\drho^{-1}
\end{equation}
which one can also check by a direction computation using
$\Pi=\phi\transp{\phi}\drho$ and $L^2(\rho)$-orthonormality of features
$(\transp{\phi}\drho \phi=\Id$). Therefore,
\begin{align}
\E_r\left[\norm{Q^{\pi_0}_r}^2_A-\norm{\hat Q-Q^{\pi_0}_r}^2_A\right]&=
\Tr(\AK_{\pi_0} \drho^{-1})
-\Tr\left(\AK_{\pi_0} (\Id-\Pi)\drho^{-1}\right)
\\&=\Tr(\AK_{\pi_0}\Pi\drho^{-1})
\\&=\Tr(\AK_{\pi_0}\phi\transp{\phi})
\\&=\Tr(\transp{\phi}\AK_{\pi_0}\phi).
\end{align}
This yields the expression of the expected gain in Theorem~\ref{thm:featurereturn} for the case of
random Gaussian or goal-reaching rewards.

For random scattered rewards, the computation has one more term. Let us
start with \eqref{eq:return_Err}. By Proposition~\ref{prop:Err}, 
\begin{equation}
\E [r\transp{r}]=
\kappa (\mu^2+\sigma^2)\drho^{-1}+
(\kappa\mu)^2 \1\transp{\1}.
\end{equation}
The first term is proportional to $\drho^{-1}$, so its contribution to
\eqref{eq:return_Err} is
the same as for random Gaussian rewards, up to the additional factor
$\kappa (\mu^2+\sigma^2)$. The contribution of the second term
$(\kappa\mu)^2 \1\transp{\1}$ to \eqref{eq:return_Err} is $(\kappa\mu)^2$
times
\begin{multline}
\Tr(\AK_{\pi_0}\1\transp{\1})-\Tr(\AK_{\pi_0}(\Id-\Pi)\1\transp{\1}\transp{(\Id-\Pi}))
\\=\transp{\1}\AK_{\pi_0}\1-\transp{((\Id-\Pi)\1)}\AK_{\pi_0}(\Id-\Pi)\1
\\=2\transp{(\Pi\1)}\AK_{\pi_0}\1-\transp{(\Pi\1)}\AK_{\pi_0}\Pi\1
\end{multline}

Now, by Lemma~\ref{lem:A1}, one has $\AK_{\pi_0}\1=0$. Therefore,
the above reduces to $-\transp{(\Pi\1)}\AK_{\pi_0}\Pi\1$, which is the
term $\transp{\phi_\mathrm{cst}}\,\AK_{\pi_0}\phi_\mathrm{cst}$ in
Theorem~\ref{thm:featurereturn}, by
definition of $\phi_\mathrm{cst}$.
This ends the proof.
\end{dem}

\begin{dem}[ of Corollary~\ref{cor:optfeatures}]
The Poincaré separation theorem states
that, given a symmetric matrix $A$ and a rank-$d$ orthogonal projector
$\Pi$, the $i$-th largest eigenvalue of $\Pi A\Pi$ is at most the $i$-th
largest eigenvalue of $A$. Consequently, the trace of $\Pi A \Pi$ is at
most the sum of the largest $d$ eigenvalues of $A$, which is achieved
when $\Pi$ coincides with the largest $d$ eigendirections of $A$.

For our case, let us perform a change of basis $\phi=\drho^{-1/2}\tilde
\phi$: the $L^2(\rho)$-orthonormality of $\phi$ is equivalent to Euclidean
orthonormality of $\tilde \phi$. Then $\tilde \Pi\deq
\tilde \phi\transp{\tilde \phi}$ is the Euclidean orthogonal projector
onto $\tilde \phi$. We can then rewrite
\begin{align}
\Tr(\transp{\phi}\AK_{\pi_0} \phi)
&=\Tr(\transp{\tilde \phi} \drho^{-1/2}\AK_{\pi_0} \drho^{-1/2} \tilde \phi)
\\&=\Tr(\drho^{-1/2}\AK_{\pi_0} \drho^{-1/2} \tilde \phi\transp{\tilde
\phi})
\\&=\Tr(\drho^{-1/2}\AK_{\pi_0} \drho^{-1/2} \tilde \Pi)
\\&=\Tr(\drho^{-1/2}\AK_{\pi_0} \drho^{-1/2} \tilde \Pi^2)
\\&=\Tr(\tilde \Pi \drho^{-1/2}\AK_{\pi_0} \drho^{-1/2}\tilde \Pi).
\end{align}

Therefore, by the Poincaré separation theorem, this is
maximal when $\tilde \phi$ are the largest $d$ eigenvectors of
$\drho^{-1/2}\AK_{\pi_0}\drho^{-1/2}$, or equivalently, when $\phi$ are
the largest $d$ eigenvectors of $\drho^{-1}\AK_{\pi_0}$. 

Note that $\drho^{-1}\AK_{\pi_0}$ is self-adjoing in $L^2(\rho)$, because
$\AK_{\pi_0}$ is symmetric. Therefore, 
the largest $d$ eigenvectors of $\drho^{-1}\AK_{\pi_0}$
are also the extrema of $\langle
r,\drho^{-1}\AK_{\pi_0}r\rangle_{L^2(\rho)}/\norm{r}^2_{L^2(\rho)}$, and since
$\langle r,\drho^{-1}\AK_{\pi_0}r\rangle_{L^2(\rho)}
=
\transp{r}\AK_{\pi_0}r
$, they are also the extrema
of $\transp{r}\AK_{\pi_0}r/\norm{r}^2_{L^2(\rho)}$.
%(also note that $\drho^{-1}\AK_{\pi_0}$ is self-adjoint in $L^2(\rho)$).
This proves the claim for the case
of random Gaussian and random goal-reaching rewards.

For the case of random scattered rewards, on top of the
$\Tr(\transp{\phi}\AK_{\pi_0} \phi)$ term, there is an additional term
$-\transp{(\Pi\1)}\AK_{\pi_0}\Pi\1\leq 0$. So, a priori, the maximum is
lower due to this
term, and might be obtained with a different choice of $\phi$. However,
since $\AK_{\pi_0}\1=0$, the main eigendirections above
do \emph{not} include the constants, and are $L^2(\rho)$-orthogonal to
the constants. Therefore, if we set
$\phi$ to those eigendirections, then $\Pi\1=0$ so the additional term is
$0$. Namely, the choice of $\phi$ that maximizes the
$\Tr(\transp{\phi}\AK_{\pi_0}\phi)$ term \emph{also} sets the extra negative term to
$0$, so the maximum is still given by those eigendirections.

This ends the proof of Theorem~\ref{thm:featurereturn}.
\end{dem}

We now turn to the proof of Theorem~\ref{thm:lap}. The first proposition
relates three norms in an MDP: the advantage norm, the Dirichlet form
$\langle f,\Delta f\rangle$, and the $L^2(\rho)$ norm.

\begin{prop}
\label{prop:Anorm}
In a deterministic environment, for any function $f$ on state-actions,
and any decay factor $0\leq \gamma\leq 1$,
\begin{align}
\norm{f}^2_A&=\norm{f}^2_{L^2(\rho)}-\norm{P_{\pi_0} f}^2_{L^2(\rho)}
\end{align}
and for $0<\gamma\leq 1$ this is further equal to
\begin{align}
\norm{f}^2_A
&=\frac1{\gamma^2}\left(
2\langle f,\Delta f\rangle_{L^2(\rho)} -\norm{\Delta f}^2_{L^2(\rho)}-(1-\gamma^2)
\norm{f}^2_{L^2(\rho)}
\right)
\end{align}
where $\Delta=\Id-\gamma P_{\pi_0}$ is the Laplacian operator of $\pi_0$.
\end{prop}

\begin{dem}
Let us denote by $\bar \rho$ the row vector of size $S$ with components
given by $\bar \rho_s\deq\rho(s)$, the marginal probability of state $s$
under $\rho$. Let $f^{\cdot 2}$ denote the pointwise square applied to a vector
$f$. Using this notation, we have
\begin{align}
\norm{f}^2_A&=
\E_{s\sim \rho,\,a\sim \pi_0(s)} (f(s,a)-\E_{a'\sim \pi_0(s)} f(s,a'))^2
\\&=
\E_{s\sim \rho,\,a\sim \pi_0(s)} f(s,a)^2 - \E_{s\sim \rho} (\E_{a\sim
\pi_0(s)} f(s,a))^2
\\&= \bar\rho \pi_0 f^{\cdot 2}-\bar\rho(\pi_0 f)^{\cdot 2}
\end{align}
where we view $\pi_0$ as an $S\times (S\times A)$ matrix, as explained in
the Notation.

Since $\rho$ is the invariant distribution of $\pi_0$, we have $\bar
\rho=\bar \rho \pi_0 P$. Therefore,
\begin{equation}
\norm{f}^2_A=\bar\rho \pi_0 f^{\cdot 2}-\bar\rho \pi_0 P (\pi_0 f)^{\cdot
2}.
\end{equation}

An environment is deterministic if and only if the variance of
$g(s_{t+1})$ knowing $(s_t,a_t)$ is $0$ for any function $g$. In other
words, the environment is deterministic if and only if $Pg^{\cdot
2}-(Pg)^{\cdot 2}=0$ for any $g$. Therefore, in a deterministic
environment,
\begin{equation}
\bar\rho \pi_0 P (\pi_0 f)^{\cdot 2}=\bar\rho \pi_0 (P\pi_0 f)^{\cdot 2}
=\bar\rho \pi_0 (P_{\pi_0} f)^{\cdot 2}.
\end{equation}

Finally, for any function $g$, $\bar\rho\pi_0 g^{\cdot 2}$ is just
another notation for $\norm{g}^2_{L^2(\rho)}$. Therefore, we find
\begin{equation}
\norm{f}^2_A=\bar\rho \pi_0 f^{\cdot 2}-\bar\rho\pi_0 (P_{\pi_0}f)^{\cdot
2}=\norm{f}^2_{L^2(\rho)}-\norm{P_{\pi_0}f}^2_{L^2(\rho)}
\end{equation}
as needed.
The second statement follows from substituting
$P_{\pi_0}=\frac{1}{\gamma}(\Id-\Delta)$ and expanding.
\end{dem}

\begin{dem}[ of Theorem~\ref{thm:lap}]
Let $r\in L^2_0(\rho)$. Since $r$ averages to $0$, its $Q$-function
$Q^{\pi_0}_r$ is well-defined for $\gamma=1$. By definition of the
advantage kernel, we have
\begin{equation}
\transp{r}
\AK_{\pi_0}r=\norm{A^{\pi_0}_r}^2_{L^2(\rho)}=\norm{Q^{\pi_0}_r}^2_A
\end{equation}
by definition of $\norm{\cdot}_A$.

By Proposition~\ref{prop:Anorm} for $\gamma=1$, we have
\begin{equation}
\norm{Q^{\pi_0}_r}^2_A=2\langle Q^{\pi_0}_r,\Delta Q^{\pi_0}_r
\rangle_{L^2(\rho)}-\norm{\Delta Q^{\pi_0}_r}^2_{L^2(\rho)}.
\end{equation}
Now, $Q^{\pi_0}_r$ satisfies the Bellman equation
$Q^{\pi_0}_r=r+P_{\pi_0} Q^{\pi_0}_r$ for $\gamma=1$. By definition of
$\Delta$, this rewrites as
\begin{equation}
\Delta Q^{\pi_0}_r=r.
\end{equation}
Therefore,
\begin{align*}
\norm{Q^{\pi_0}_r}^2_A
&=2 \langle Q^{\pi_0}_r,r\rangle_{L^2(\rho)}-\norm{r}^2_{L^2(\rho)}
\\&=2\langle \Delta^{-1} r,r\rangle_{L^2(\rho)}-\norm{r}^2_{L^2(\rho)}
\\&=\langle \Delta^{-1} r,r\rangle_{L^2(\rho)}+\langle
r,(\Delta^{-1})^{\ast} r\rangle_{L^2(\rho)}-\norm{r}^2_{L^2(\rho)} 
\\&=\langle r, (\Delta^{-1}+(\Delta^{-1})^{\ast}-\Id)r\rangle_{L^2(\rho)}
\end{align*}
as needed.

Therefore, the functions $r\in L^2_0(\rho)$ that are the extrema of
$\transp{r}\AK_{\pi_0}r/\norm{r}^2_{L^2(\rho)}$ are the largest
eigenfunctions of $\Delta^{-1}+(\Delta^{-1})^{\ast}-\Id$, or equivalently
the
largest eigenfunctions of $\Delta^{-1}+(\Delta^{-1})^{\ast}$.
\end{dem}

\begin{dem}[ of Proposition~\ref{prop:gamma0}]
Let $r\in L^2(\rho)$. By definition of the
advantage kernel, we have
\begin{equation}
\transp{r}
\AK_{\pi_0}r=\norm{A^{\pi_0}_r}^2_{L^2(\rho)}=\norm{Q^{\pi_0}_r}^2_A
\end{equation}
by definition of $\norm{\cdot}_A$.

But for $\gamma=0$ we have $Q^{\pi_0}_r=r$, so
\begin{equation}
\transp{r}
\AK_{\pi_0}r=\norm{r}^2_A.
\end{equation}

By Proposition~\ref{prop:Anorm} for $\gamma=0$, we have
\begin{align}
\norm{r}^2_A&=\norm{r}^2_{L^2(\rho)}-\norm{P_{\pi_0} r}^2_{L^2(\rho)}
\\&=\langle r,r\rangle_{L^2(\rho)}
-\langle P_{\pi_0} r,P_{\pi_0} r\rangle_{L^2(\rho)}
\\&=\langle r,(\Id-P_{\pi_0}^\ast P_{\pi_0}) r\rangle_{L^2(\rho)}
\end{align}
as needed.

Therefore, the extrema of $\transp{r}\AK_{\pi_0}r/\norm{r}^2_{L^2(\rho)}$
are the largest eigenfunctions of $\Id-P_{\pi_0}^\ast P_{\pi_0}$, namely,
the smallest eigenfunctions of $P_{\pi_0}^\ast P_{\pi_0}$.
\end{dem}

\begin{dem}[ of Theorem~\ref{thm:gen}]
Let $r\in L^2(\rho)$. By definition of the
advantage kernel, we have
\begin{equation}
\transp{r}
\AK_{\pi_0}r=\norm{A^{\pi_0}_r}^2_{L^2(\rho)}=\norm{Q^{\pi_0}_r}^2_A
\end{equation}
by definition of $\norm{\cdot}_A$.

By Proposition~\ref{prop:Anorm} for $0<\gamma<1$, we have
\begin{equation}
\gamma^2 \norm{Q^{\pi_0}_r}^2_A=2\langle Q^{\pi_0}_r,\Delta Q^{\pi_0}_r
\rangle_{L^2(\rho)}-\norm{\Delta Q^{\pi_0}_r}^2_{L^2(\rho)}-(1-\gamma^2)
\norm{Q^{\pi_0}_r}^2_{L^2(\rho)}.
\end{equation}
Now, $Q^{\pi_0}_r$ satisfies the Bellman equation
$Q^{\pi_0}_r=r+\gamma P_{\pi_0} Q^{\pi_0}_r$. By definition of
$\Delta$, this rewrites as
\begin{equation}
\Delta Q^{\pi_0}_r=r, \qquad Q^{\pi_0}_r=\Delta^{-1}r.
\end{equation}
Therefore,
\begin{align*}
\gamma^2 \norm{Q^{\pi_0}_r}^2_A
&=2 \langle
Q^{\pi_0}_r,r\rangle_{L^2(\rho)}-\norm{r}^2_{L^2(\rho)}-(1-\gamma^2)\norm{Q^{\pi_0}_r}^2_{L^2(\rho)}
\\&=2\langle \Delta^{-1}
r,r\rangle_{L^2(\rho)}-\norm{r}^2_{L^2(\rho)}-(1-\gamma^2)\norm{\Delta^{-1}r}^2_{L^2(\rho)}
% \\&=\langle \Delta^{-1} r,r\rangle_{L^2(\rho)}+\langle
% r,(\Delta^{-1})^{\ast} r\rangle_{L^2(\rho)}-\norm{r}^2_{L^2(\rho)} 
% -(1-\gamma^2) \langle r, (\Delta^{-1})^\ast \Delta^{-1} r\rangle_{L^2(\rho)}
\\&=\langle r, \left(\Delta^{-1}+(\Delta^{-1})^{\ast}-\Id-(1-\gamma^2)
(\Delta^{-1})^\ast \Delta^{-1}\right)r\rangle_{L^2(\rho)}
\end{align*}
as needed.

Therefore, the extrema of
$\transp{r}\AK_{\pi_0}r/\norm{r}^2_{L^2(\rho)}$ are the largest
eigenfunctions of $\Delta^{-1}+(\Delta^{-1})^{\ast}-\Id-(1-\gamma^2)
(\Delta^{-1})^\ast \Delta^{-1}$, or equivalently
the
largest eigenfunctions of $\Delta^{-1}+(\Delta^{-1})^{\ast}-(1-\gamma^2)
(\Delta^{-1})^\ast \Delta^{-1}$.

Note that $\Delta^{-1}+(\Delta^{-1})^{\ast}-\Id-(1-\gamma^2)
(\Delta^{-1})^\ast \Delta^{-1}$ is a non-negative operator, since
$\norm{Q^{\pi_0}_r}^2_A\geq 0$.

\bigskip

Alternatively, one can start with the first expression in
Proposition~\ref{prop:Anorm}, namely $\norm{Q^{\pi_0}_r}^2_A=\norm{Q^{\pi_0}_r}^2_{L^2(\rho)}-
\norm{P_{\pi_0} Q^{\pi_0}_r}^2_{L^2(\rho)}$. By a similar derivation,
this leads to a slightly different expression for the same quantities,
mixing $\Delta$ and $P_{\pi_0}$:
\begin{equation}
\transp{r}\AK_{\pi_0}r=
\langle r, (\Delta^{-1})^\ast (\Id-P_{\pi_0}^\ast
P_{\pi_0})\Delta^{-1} r\rangle_{L^2(\rho)}.
\end{equation}
\end{dem}

\begin{dem}[ of Proposition~\ref{prop:gap}]
By definition, successor features estimate $\hat Q$ as
the $Q$-function of the estimated reward $\hat r$,
\begin{equation}
\hat Q=\hat r+ \gamma P_{\pi_0} \hat Q
\end{equation}
where
\begin{equation}
\hat r(s,a)=\transp{z}\phi(s,a)
\end{equation}
is the $L^2(\rho)$-orthogonal projection of $r$ onto the features $\phi$,
with $z$ given by \eqref{eq:zrsf}. Therefore
\begin{equation}
\norm{\hat Q-r-\gamma P_{\pi_0} \hat Q}^2_{L^2(\rho)}=\norm{\hat
r-r}^2_{L^2(\rho)}=\norm{(\Id-\Pi)r}^2_{L^2(\rho)}
\end{equation}
where $\Pi$ is the $L^2(\rho)$-orthogonal projector onto the features.

Therefore, we have
\begin{align}
\E_r \norm{\hat Q-r-\gamma P_{\pi_0} \hat Q}^2_{L^2(\rho)}
&=
\E_r \norm{(\Id-\Pi)r}^2_{L^2(\rho)}
\\&=\E_r \left[
\langle (\Id-\Pi)r, (\Id-\Pi)r \rangle_{L^2(\rho)}\right]
\\&=\E_r \left[\langle r, (\Id-\Pi)r \rangle_{L^2(\rho)}\right]
\\&=\E_r \left[\transp{r} \drho (\Id-\Pi)r\right]
\\&=\E_r \Tr\left(
r\transp{r}\drho (\Id-\Pi)
\right)
\\&= \Tr\left( \E_r[r\transp{r}] \drho (\Id-\Pi)\right).
\end{align}

By Proposition~\ref{prop:Err}, we have $\E_r[r\transp{r}]=\drho^{-1}$ in
the random Gaussian reward and random goal-reaching reward models.
Therefore the above is
\begin{align}
\Tr\left( \E_r[r\transp{r}] \drho (\Id-\Pi)\right)
&= \Tr\left( \Id-\Pi\right)
\\&=\#S\times \#A - d
\end{align}
since $\Pi$ is a projector of rank $d$. This proves the result.

The same proof works for the universal successor features as in
\cite{borsa2018universal}, where $\hat Q$
satisfies a Bellman equation with respect to $\pi_z$ instead of $\pi_0$.
\end{dem}

\newpage

\bibliographystyle{alpha}
\bibliography{whylowfreq}

\end{document}

%% file: header.en.tex
\usepackage{cmap}
\usepackage[utf8x]{inputenc}
\usepackage[T1]{fontenc}
\usepackage[english]{babel}
\usepackage{amssymb,amsmath}
\usepackage{amsthm}
\usepackage{bbm}
\usepackage{microtype}
\usepackage{lmodern}

{
\rmfamily
\DeclareFontShape{T1}{lmr}{m}{sc}{<->ssub*cmr/m/sc}{}
\DeclareFontShape{T1}{lmr}{b}{sc}{<->ssub*cmr/b/sc}{}
\DeclareFontShape{T1}{lmr}{bx}{sc}{<->ssub*cmr/bx/sc}{}
}

\input{header.tex}

\newenvironment{dem}[1][]{\begin{proof}[\thmheadercommand{Proof#1}]~\newline\ignorespaces}{\end{proof}}

{
\theoremstyle{yannthm}
\newtheorem{defi}{Definition}
\newtheorem*{defi*}{Definition}
\newtheorem{prop}[defi]{Proposition}
\newtheorem*{prop*}{Proposition}
\newtheorem{thm}[defi]{Theorem}
\newtheorem*{thm*}{Theorem}
\newtheorem{lem}[defi]{Lemma}
\newtheorem*{lem*}{Lemma}
\newtheorem{cor}[defi]{Corollary}
\newtheorem*{cor*}{Corollary}

\newtheorem*{ex*}{Example}

\newtheorem*{subenonce}{}

\theoremstyle{yannthm2}

\newtheorem*{exo*}{Exercise}

\newtheorem*{rem*}{Remark}

\newtheorem*{subenonce2}{}

}

\newcommand{\transp}[1]{#1^{\!\top}\!}

%% file: header.tex
\newcommand{\thmheadercommand}[1]{\textbf{\scshape{}#1.\\*}}

\newtheoremstyle{yannthm}{\topsep}{\topsep}{\slshape}{}{\scshape\bfseries}{.}{.5em}{%
\thmname{#1}\thmnumber{ #2}\thmnote{#3}%
}
\newtheoremstyle{yannthm2}{\topsep}{\topsep}{}{}{\scshape\bfseries}{.}{.5em}{%
\thmname{#1}\thmnumber{ #2}\thmnote{#3}%
}

\def\d{\operatorname{d}\!{}}

\def\R{{\mathbb{R}}}

\renewcommand{\geq}{\geqslant}
\renewcommand{\leq}{\leqslant}

\newcommand{\deq}{\mathrel{\mathop{:}}=}

\newcommand{\from}{\colon} % correct ':' in f\from X \to Y
\def\eps{\varepsilon}
\renewcommand{\epsilon}{\varepsilon}
\renewcommand{\phi}{\varphi}

\DeclareMathOperator{\Var}{Var}
\DeclareMathOperator{\Cov}{Cov}
\let\oldPr\Pr
\renewcommand{\Pr}{\oldPr\nolimits}
\newcommand{\E}{\mathbb{E}}
\newcommand{\KL}[2]{\mathrm{KL}\!\left(#1 \,|\hspace{-.15ex}|\,#2\right)}

\DeclareMathOperator{\Tr}{Tr}

\DeclareMathOperator{\Id}{Id}

\DeclareMathOperator{\diag}{diag}

\newcommand{\norm}[1]{\left\lVert#1\right\rVert}

\newcommand{\1}{\mathbbm{1}}